\documentclass[sigconf]{acmart}
\AtBeginDocument{%
  }

\setcopyright{acmlicensed}
\copyrightyear{2018}
\acmYear{2018}
\acmDOI{XXXXXXX.XXXXXXX}
\acmConference[Conference acronym 'XX]{Make sure to enter the correct
  conference title from your rights confirmation email}{June 03--05,
  2018}{Woodstock, NY}
\acmISBN{978-1-4503-XXXX-X/2018/06}

\usepackage{amsmath,amssymb,amsthm}
\usepackage{subcaption}
\usepackage{bbm}
\newtheorem{theorem}{Theorem}

\usepackage{booktabs}
\usepackage{multirow}
\usepackage{subcaption}
\usepackage{xcolor}               
\newcommand{\ulnum}[1]{\underline{#1}}

\begin{document}

\title{Beyond Normality: Reliable A/B Testing with Non-Gaussian Data}


\author{Junpeng Gong}
\affiliation{%
  \institution{School of Mathematical Sciences, University of Chinese Academy of Sciences
  \city{Beijing}   
  \country{China}}
  \institution{SKLMS, Academy of
Mathematics and Systems Science, Chinese Academy of Sciences
\city{Beijing}   
\country{China}}
}
\email{gongjunpeng@amss.ac.cn}

\author{Chunkai Wang}\affiliation{\institution{ByteDance}\city{Beijing}   \country{China}}
\email{wangjunkai.70726@bytedance.com}

\author{Hao Li}\affiliation{\institution{ByteDance}\city{Beijing}   \country{China}}
\email{lihao.sky@bytedance.com}

\author{Jinyong Ma}\affiliation{\institution{ByteDance}\city{Beijing}   \country{China}}
\email{jinyongma@bytedance.com}

\author{Haoxuan Li}\affiliation{\institution{Peking University}\city{Beijing}   \country{China}}
\email{hxli@stu.pku.edu.cn}

\author{Xu He}\affiliation{\institution{SKLMS, Academy of
Mathematics and Systems Science, Chinese Academy of Sciences\city{Beijing}   \country{China}}}
\email{hexu@amss.ac.cn}

\begin{abstract}
A/B testing has become the cornerstone of decision-making in online markets, guiding how platforms launch new features, optimize pricing strategies, and improve user experience. In practice, we typically employ the pairwise $t$-test to compare outcomes between the treatment and control groups, thereby assessing the effectiveness of a given strategy. To be trustworthy, these experiments must keep Type I error (i.e., false positive rate) under control; otherwise, we may launch harmful strategies. However, in real-world applications, we find that A/B testing often fails to deliver reliable results. When the data distribution departs from normality or when the treatment and control groups differ in sample size, the commonly used pairwise $t$-test is no longer trustworthy. In this paper, we quantify how skewed, long tailed data and unequal allocation distort error rates and derive explicit formulas for the minimum sample size required for the $t$-test to remain valid. We find that many online feedback metrics require hundreds of millions samples to ensure reliable A/B testing. Thus we introduce an Edgeworth-based correction that provides more accurate $p$-values when the available sample size is limited. Offline experiments on a leading A/B testing platform corroborate the practical value of our theoretical minimum sample size thresholds and demonstrate that the corrected method substantially improves the reliability of A/B testing in real-world conditions.
\end{abstract}



\maketitle

\section{Introduction}

Online controlled experiments, or A/B testing, are the gold standard for evaluating new product ideas \citep{2018SQRLink,XieAurisset2016KDD}, ranking algorithms \citep{2013Adclick}, pricing policies, and recommendation strategies in online markets \citep{kohavi2020trustworthy,liu21link}.
They have been widely adopted by leading technology companies, including Google, Facebook, Microsoft, LinkedIn, and ByteDance \citep{deng16micro,xu15link,XieAurisset2016KDD}.
The validity of A/B testing depends on an accurate $p$-value, the probability of observing a result as or more extreme than the current one under the null hypothesis \citep{lehmann2005pvalue}.
In online experimentation, decisions rely on comparing the $p$-value with a predefined significance level, making its accuracy essential for trustworthy inference and reliable decision making in online markets \citep{Johari17}. 

The two sample $t$-test \citep{welch} is widely used in A/B testing to compare the effects of a treatment and a control strategy and to support decision-making.
First, the experiment platform collects user-level feedback and computes a t-statistic that summarizes the difference between the treatment and control groups. Second, it evaluates how likely it would be to observe a result at least as extreme as the one seen if the null hypothesis (i.e., no difference between two groups) were true. Then, we reject the null when the reported $p$-value is at most $\alpha$, which is set in advance. If the estimated mean for treatment exceeds the mean for control, the effect is interpreted as positive and the platform may launch the new strategy; otherwise, the strategy is not launched. 

In practice, the reliability of the pairwise $t$-test heavily depends on the data distribution. Specifically, the results are valid only if the distribution is asymptotically normal \citep{hall2004exact}. Otherwise, we will obtain an incorrect false positive probability (Type I error rate) that departs from the ideal ones \citep{Kohavi10unexpected}.
For a two-sided test, the Type~I error rate should be symmetrically split between the two tails, with each tail accounting for approximately half of the nominal significance level~$\alpha$.
When the right-tail Type~I error, that is, the probability of falsely approving a positive effect, exceeds half of the nominal level, the platform becomes overly permissive, increasing the risk of launching ineffective or even harmful strategies and degrading user experience.
Conversely, when the right-tail Type~I error is below half of the nominal level, promising strategies may be prematurely discarded.

To confirm the presence of non-normality, we conduct a motivating experiment using real-world data. Specifically, we repeatedly draw independent and identically distributed from the same dataset to form the control and treatment groups, thereby ensuring that the null hypothesis holds. We then compute the test statistic $T$ over $10^6$ replications to obtain its empirical distribution. 
As shown in Figure~\ref{fig:results-density}, the empirical distribution of the test statistic $T$ is asymmetric and deviates from normality. The right-tail rejection probability (orange) exceeds its nominal $\alpha/2$, while the left-tail rejection probability (blue) falls below it. In such cases, the test becomes systematically biased, more likely to approve ineffective strategies, and thus leads to wasted resources and suboptimal decisions. 
\begin{figure}[t]
\includegraphics[width=0.8\linewidth]{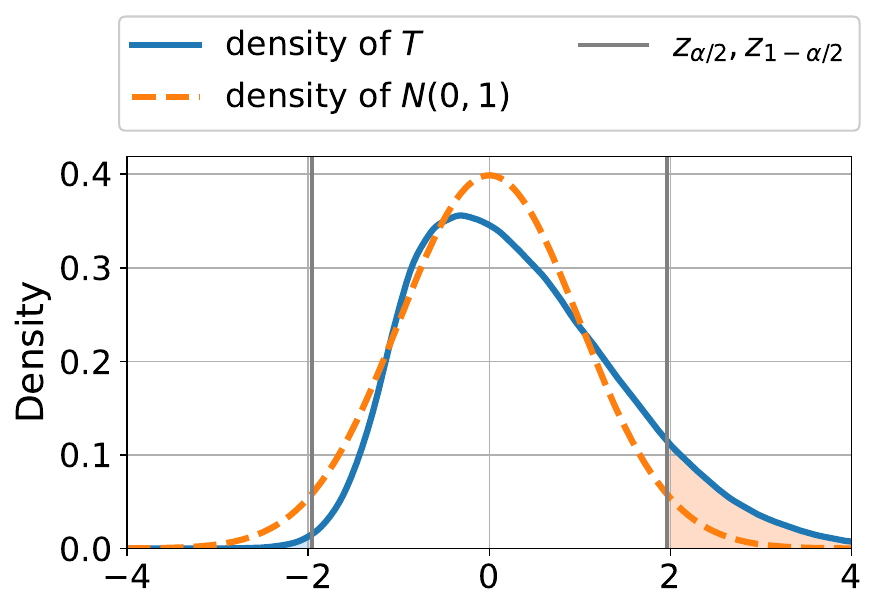}
\caption{Empirical density of the Welch's $t$-test statistic $T$ (blue), based on real data on the number of videos published by users of a popular online platform, with significance level $\alpha = 0.05$. The control and treatment sample sizes are 100 and 1,000, and the empirical distribution is compared with the standard normal $N(0,1)$ density (orange).}
\label{fig:results-density}
\end{figure}
Such deviations from normality are also observed in commonly used online metrics, including click-through rate, the number of videos published by users, and engagement duration.
The deviation mainly arises because these metrics are dominated by a large majority of low-activity users and a small minority with extremely high values.

To examine the convergence behavior of the $t$-test under such skewed conditions, \citet{Boos} conducted a systematic investigation to identify the minimum sample size required for accurate inference.
Through extensive simulations with common skewed distributions such as Gamma and Weibull, they quantified how the empirical Type~I error rate deviates from the nominal level across different degrees of skewness.
They established that the deviation in each tail is approximately proportional to the skewness divided by $\sqrt{n}$, where $n$ denotes the sample size. Finally, they derived a formula for the minimum sample size that maintains the Type~I error deviation within a prespecified tolerance.

Nevertheless, the formulas  were obtained empirically through simulation and regression, and thus lack general validity across diverse data distributions \citet{Boos}.
More importantly, online experiments introduce an additional factor that strongly influences the error rate, the imbalance in sample sizes between the treatment and control groups, i.e., unequal allocation.
Unequal allocation is common for two main reasons.
Firstly, in large scale A/B testing, platforms often assign fewer users to new or high risk strategies to minimize potential revenue loss and avoid negative user experiences. 
Second, local regulations or operational policies may cap the fraction of users exposed to treatment, leading to unequal allocation.
These imbalanced allocations can further distort the actual Type I error rate \citep{16unequal}. 
As a result, the empirical rules cannot accurately capture how the required sample size depends on the allocation ratio \citep{Boos}, which is a critical factor in modern online experimentation.


In response, this paper develops a unified theoretical framework for reliable A/B testing in general online experimental settings. We systematically analyze how the non-normality of the test statistic depend on the data distribution and traffic allocation. In particular, we establish how the skewness and kurtosis of the test statistic, which measure asymmetry and tail heaviness, are determined by the skewness and kurtosis of the treatment and control groups together with their allocation ratio. Building on this foundation, we derive a new formula for the minimum sample size required for accurate testing.  In practice, using real-world data, we find that many online metrics require sample sizes on the order of hundreds of millions to achieve reliable inference.
Therefore, we employ higher-order Edgeworth expansions \citep{hall2013bootstrap}, incorporating third- and fourth-order moments together with the allocation ratio to correct the $p$-value.

The main contributions of this paper are summarized as follows.
\begin{itemize}
\item We establish a closed-form relationship showing how data skewness and traffic imbalance jointly affect the departure of the test statistic from normality.
\item Using this formula, we analytically determine the minimum sample size needed to maintain accurate Type I error control, providing clear guidance for reliable online experimentation.
\item We introduce an Edgeworth based correction that produces more accurate $p$-value estimation in small samples, thereby supporting better decision making in online platforms. 
\end{itemize}

\section{Related Works}

\subsection{Non-Normality of Test Statistics in Online Controlled Experiments}

Online platforms rely on large-scale controlled experiments to guide product launches, ranking policies, pricing decisions, and recommendation strategies \citep{2018SQRLink,2013Adclick,XieAurisset2016KDD}. 
However, in practice, A/B testing implemented on online platforms often exhibits systematic deviations from nominal Type I error control.
\citet{Kohavi10unexpected} demonstrate through extensive A/A tests that even when no treatment effect exists, the false positive rate can reach up to 30\%, far exceeding the nominal 5\% level. 
\citet{larsen2024statistical} further emphasize that robust error rate control remains a central issue for trustworthy inference in large-scale experimentation systems. Recent studies have further examined the underlying causes of such deviations.
In particular, \citet{metric21typeI} find that metrics in online marketplaces are frequently skewed or long-tailed, causing the sampling distribution of test statistics to deviate markedly from normality and resulting in inaccuracies for normal-approximation–based test.
Furthermore, simulation evidence suggests that unequal allocation between treatment and control groups can aggravate this non-normality \citep{16unequal}.
Despite these advances, how data distribution and unequal allocation jointly affect the departure of test statistics from normality remains analytically underexplored.

\subsection{Reliable Testing for Online Controlled Experiments}
To ensure reliable hypothesis testing, numerous studies have examined different aspects of experimental validity. Variance reduction methods, such as regression adjustment and CUPED with pre-experiment co-variates, are designed to reduce metric variance and thereby accelerate convergence toward a normal limit \citep{deng2013improving}. Research on interference under increasing allocation investigates how gradually adding users affects validity and sensitivity \citep{increase23}. \citet{samplesize23} study the sample size required to maintain sufficient testing power for detecting small effects. Sequential and always-valid procedures maintain valid inference under continuous monitoring \citep{or22}. However, these directions primarily focus on controlling Type II error, which concerns ensuring adequate test power, rather than on the accuracy of Type I error control.
A complementary line of work proposes alternatives to the $t$-test. Resampling and randomization methods, including the bootstrap \citep{hall2013bootstrap} and permutation tests, are widely used to improve inference in small sample settings, although they can be computationally expensive at web scale.
\citet{Johnson1978} proposed modified $t$ statistics based on the Cornish–Fisher expansion to improve inference under asymmetry data, while \citet{Sutton1993} introduced composite decision rules that integrate bootstrap adjustments. \citet{chenlin1995} developed an Edgeworth-based test for upper-tailed inference that outperformed earlier approaches in small samples.
Foundational work further established the theoretical connections between Edgeworth expansions and bootstrap methods \citep{hall2013bootstrap}.
Collectively, this literature highlights that skewness can severely distort the normal approximation, but existing contributions are confined to one-sample contexts and do not address the additional complications posed by unequal allocation in two-sample designs.
\citet{Boos} investigates the finite-sample bias of the two-sample $t$-test and provides an empirical formula for the required sample size as a function of distributional asymmetry. Yet, this research does not consider unequal sample sizes between treatment and control groups.
Taken together, these findings highlight the need to examine and improve the accuracy of A/B testing in online experiments with skewed, long-tailed data and unequal allocation.

\section{Preliminaries}

\subsection{Two-Sample \texorpdfstring{$t$}{t}-Tests in A/B testing} 

A common way to evaluate an intervention in online experiments is to compare the average outcomes between a control group and a treatment group. This simple two-sample design forms the foundation of most statistical analyses in large-scale A/B testing.
We consider two independent groups, the control group $X_1,\cdots,X_{n_x}$ and a treatment group $Y_1,\cdots,Y_{n_y}$, with total sample size 
$N=n_x+n_y$. 
The population means of the two groups are $\mu_x=E(X_i)$, $\mu_y=E(Y_j)$ with corresponding sample means
$$\bar X = n_x^{-1}\sum_{i=1}^{n_x} X_i, \quad \bar Y = n_y^{-1}\sum_{j=1}^{n_y} Y_j.$$
The null hypothesis of interest is the equality of population means, which is $
H_0:\ \mu_x=\mu_y $.
In modern experimentation platforms, the population variances are typically unknown 
and may differ across groups. Denote the
variances be $\sigma_x^2 = E(X_i - \mu_x)^2$, $\sigma_y^2 = E(Y_j - \mu_y)^2$ with sample estimates
$$\hat\sigma_x^2 =  n_x^{-1} \sum_{i=1}^{n_x}(X_i-\bar X)^2, \quad \hat\sigma_y^2 =  n_y^{-1} \sum_{j=1}^{n_y}(Y_j-\bar Y)^2.$$
The standard approach for comparing two groups with possibly unequal variances is Welch’s $t$-test \citep{welch}, with the test statistic given by
\begin{equation*}
\label{equ:tstas}
T = \frac{\bar Y-\bar X}{\sqrt{\hat\sigma_x^2/n_x+\hat\sigma_y^2/n_y}} .
\end{equation*}
By the central limit theorem, $T$ is asymptotically normal under $H_0$, and inference 
is conducted by approximating its distribution with the standard normal. The two–sided 
$p$–value of $t$-test is therefore
\[
p_t = 2\bigl(1-\Phi(|T|)\bigr),
\]
where $\Phi(x)$ and $\phi(x)$ denote the cumulative and probability density functions of the standard normal distribution $N(0,1)$. 
The standard testing procedure compares the $p$ value with the nominal significance level $\alpha \in (0,1)$, which is predetermined before the experiment. If the $p$ value is smaller than $\alpha$, the null hypothesis is rejected, indicating that the means of the treatment and control groups are not equal.

In online platforms, the direction of rejection is often of primary interest. The objective is not only to determine whether the two means differ but also to identify which variant is beneficial. When the $p$ value is smaller than $\alpha$ and the test statistic $T$ is positive, we infer that $\mu_y > \mu_x$, indicating that the treatment outperforms the control and may thus be deployed to users. Conversely, when the $p$ value is smaller than $\alpha$ and $T$ is negative, we infer that $\mu_y < \mu_x$, suggesting that the control version yields better outcomes and should be retained. If the $p$ value is greater than or equal to $\alpha$, no statistically significant difference is detected, and the final decision may depend on auxiliary factors such as implementation cost or operational risk.
Each rejection direction carries distinct business implications for online platforms, underscoring the importance of maintaining accurate and well-balanced Type I error control across both tails.

\subsection{Error Metric and Problem Statement}

From the preceding discussion, we have described how the test operates in practice. In this section, we formally define the rejection probability and the deviation of Type I error rate.
Under the null hypothesis $H_0$, the actual Type I error of $t$-test is
\[
\alpha^t \;=\; \Pr(p_t  < \alpha) 
= \Pr\bigl(|T| >  z_{1-\alpha/2}\bigr),
\]
where $z_\beta=\Phi^{-1}(\beta)$ for the $\beta$ quantile of $N(0,1)$ with $\beta\in(0,1)$. 
and the deviation from the nominal level is measured by
\[
\Delta(\alpha) = \bigl|\alpha^t -\alpha\bigr|
= \left|\Pr(|T| > z_{1-\alpha/2})-\alpha\right|.
\]
While $\Delta(\alpha)$ summarizes the overall deviation of the Type I error from its nominal level $\alpha$, it does not reveal how this error is distributed between the two tails of the test. In practice, such information is essential because the direction of rejection carries distinct operational meaning in online experimentation.  
We therefore define the left- and right-tail rejection probabilities as
$$
\alpha_L^t = \Pr(p_t < \alpha, T<0) =  \Pr(T < z_{\alpha/2}),$$
$$
\alpha_R^t = \Pr(p_t < \alpha, T>0) = \Pr(T > z_{1-\alpha/2}).
$$
A reliable A/B testing procedure should allocate the total Type I error approximately evenly between the two tails, that is, $\alpha_L^t \approx \alpha_R^t \approx \alpha/2$. 
When the error allocation is unbalanced, the test becomes biased toward one direction of effect, leading to misleading conclusions. For example, when $\alpha_R^t$ is greater than $\alpha/2$, the test becomes more prone to incorrectly concluding that an ineffective treatment is beneficial. This bias can lead to the premature rollout of unproven strategies, resulting in unnecessary business costs. In large-scale online platforms, such asymmetry can also explain why A/A experiments sometimes exhibit uneven false positive rates across the two tails. These inconsistencies undermine the perceived reliability of the A/B testing system and may reduce confidence in its experimental results.

Therefore, in this paper, we aim to control both the total deviation of the Type I error and its directional balance between the two tails.
Given a predetermined tolerance $\epsilon > 0$, we say that the two-sample $t$-test achieves reliable inference if the deviations of the Type I error in both tails are within this tolerance, that is,
\[
\max\{\alpha_L^t - \alpha/2,\alpha_R^t - \alpha/2\} \le \epsilon.
\]
This condition ensures that the overall Type I error remains close to its nominal level $\alpha$ and that the error is allocated nearly symmetrically between the two tails.

\section{Methodology}

\subsection{Finite-Sample Error in A/B Testing}
In this section, we quantify the deviation of the Type~I error, both in total and for each tail separately.
The leading departures from normality are largely governed by the skewness and kurtosis of the underlying outcome distributions \citep{hall2013bootstrap}.
For the random variable $X_i$, the skewness and kurtosis are 
\begin{equation*}
\gamma_x= \sigma_x^{-3} E(X_i-\mu_x)^3,\quad
\tau_x=\sigma_y^{-4} E(X_i-\mu_x)^4,
\end{equation*}
with $\gamma_y,\tau_y$ defined analogously. 
Skewness characterizes the asymmetry of the data, whereas kurtosis describes the heaviness of its tails. For the standard normal distribution $N(0,1)$, the skewness is zero and the kurtosis is three. 

The following theorem presents the skewness and kurtosis of the mean difference $D = \bar{Y} - \bar{X}$, which determine the rate at which the test statistic $T$ approaches its asymptotic normal distribution.
\begin{theorem}
\label{thm:cumulants}
Suppose $X_1,\ldots,X_{n_x}$ are independent and identically distributed with $E|X_1|^4<\infty$, and $Y_1,\ldots,Y_{n_y}$ are independent and identically distributed with $E|Y_1|^4<\infty$.
Assume that the two samples ${X_i}$ and ${Y_j}$ are independent.
Then, for the mean difference $D = \bar{Y} - \bar{X}$, the skewness and kurtosis are given by
\begin{align}\label{equ:gammad}
    \gamma_D &=  \frac{(1+k)^{1/2}}{\sqrt{Nk}}\frac{\gamma_y \sigma_y^3 - k^2\gamma_x\sigma_x^3}{(k\sigma_x^2 + \sigma_y^2)^{3/2}},\\
    \tau_D &= 3 + \frac{1+k}{kN}\frac{(\tau_y-3)\sigma_y^4  + k^3(\tau_x-3)\sigma_x^4}{\bigl(\sigma_y^2+k\sigma_x^2\bigr)^2}.
\end{align}
\end{theorem}
This result shows how the asymmetry and tail behavior of the original data distributions propagate to the sampling distribution of the mean difference.
In particular, the skewness of the mean difference $\gamma_D$ vanishes when both groups are symmetric ($\gamma_x = \gamma_y = 0$), or when the experiment is equally allocated with identical skewness and variance ($k = 1$, $\gamma_x = \gamma_y$, $\sigma_x = \sigma_y$).
In general, $\gamma_D$ decreases at the rate $O(N^{-1/2})$ and $\tau_D -3$ decreases at the rate $O(N^{-1})$.
Let $i$ denote the imaginary unit, and let $\varphi_{X_1}(t)=E[e^{itX_1}]$ and $\varphi_{Y_1}(t)=E[e^{itY_1}]$) be the characteristic functions of $X_i$ and $Y_j$.  
To quantify the deviations of type I error precisely, we establish the following theorem.

\begin{theorem}
\label{thm:deviation}
Under the conditions of Theorem~\ref{thm:cumulants}, and additionally assuming the Cramér condition holds where 
$\limsup_{|t|\to\infty}|\varphi_{X_1}(t)|<1$ and $\limsup_{|t|\to\infty}|\varphi_{Y_1}(t)|<1$. 
As $n_x, n_y \rightarrow \infty$, we have 
\begin{align*}
\alpha^*_L - \alpha/2 &= \frac{\gamma_D}{6}\phi(z_{\alpha/2})(2z_{\alpha/2}^2+1) + O(N^{-1}),\\
\alpha^*_R - \alpha/2 &= -\frac{\gamma_D}{6}\phi(z_{\alpha/2})(2z_{\alpha/2}^2+1) + O(N^{-1}),
\end{align*}
and thus we have the total deviation is 
$$\Delta(\alpha)=O(N^{-1}).$$
\end{theorem}

These results show that the left and right tails deviate in opposite directions at order $O(N^{-1/2})$, with the magnitude determined by $\gamma_D$.
The sign of $\gamma_D$ in equation~\eqref{equ:gammad} further indicates which tail exhibits inflation.
When $\gamma_D>0$, the left-tail probability satisfies $\alpha_L^t>\alpha/2$ while the right-tail probability satisfies $\alpha_R^t<\alpha/2$; when $\gamma_D<0$, the inequalities reverse.
This directional property is practically important in online experimentation, as it reveals whether the test is more prone to false approvals or false rejections under data skewness or allocation imbalance.
Moreover, since the two tail deviations have opposite signs, their first-order effects cancel, yielding an overall two-sided deviation of order $O(N^{-1})$.
This explains why the total Type~I error can look well controlled even when the test remains directionally biased. Hence, controlling only the total error is not sufficient; balanced tail behavior is also essential for valid and trustworthy inference at scale.

\subsection{Minimum Sample Size for Reliable $t$-test}

In this subsection we analyze the size error of the two sample $t$-test for A/B testing and derive explicit thresholds for the total sample size $N$ that guarantee reliable inference at a prespecified tail tolerance $\epsilon>0$. Let $\operatorname{sign}(x)=1$ if $x\ge0$ and $-1$ otherwise, and let $\min\{x,y\}$ and $\max\{x,y\}$ denote the smaller and larger of two real numbers $x$ and $y$, respectively.

\begin{theorem}
\label{thm:threshold}
Under the conditions of Theorem~\ref{thm:deviation}, let
\begin{equation*}
a_1 = \frac{1}{6}\,(2z_{\alpha/2}^2+1)\,\phi(z_{\alpha/2})
       \sqrt{\tfrac{1+k}{k}}\,
       \frac{\gamma_y \sigma_y^3 - k^2 \gamma_x \sigma_x^3}
            {(k\sigma_x^2 + \sigma_y^2)^{3/2}}.
\end{equation*}
For $a_1\neq 0$ and $\epsilon >0$,
(1) if the total sample size satisfies
\begin{equation}
\label{equ:Nmin1}
N\geq N_{\min}^{(1)} = (a_1/\epsilon)^2,
\end{equation}
then the maximal tail deviation is controlled at the level
\begin{equation*}
\max\{|\alpha_L^* - \alpha/2|,|\alpha_R^* - \alpha/2| \} \leq \epsilon + O(\epsilon^2).
\end{equation*}
\noindent (2)
Further assume that sixth moments exist, that is 
$E|X_1|^6<\infty$ and $E|Y_1|^6<\infty$. Let 
\begin{align*}
a_2 &= \phi(z_{\alpha/2})
\left\{
\frac{1+k}{12k}
\frac{(\tau_y-3)\sigma_y^4 + k^3(\tau_x-3)\sigma_x^4}
     {(\sigma_y^2+k\sigma_x^2)^2}
\,(z_{\alpha/2}^3 - 3z_{\alpha/2}) \right. \\[6pt]
&\quad - \frac{1+k}{18k}
\frac{(\gamma_y\sigma_y^3 - k^2\gamma_x\sigma_x^3)^2}
     {(k\sigma_x^2 + \sigma_y^2)^3}
\,(z_{\alpha/2}^5 + 2z_{\alpha/2}^3 - 3z_{\alpha/2}) \\[6pt]
&\quad \left. - \frac{1+k}{4}
\frac{(k^3\sigma_x^4+\sigma_y^4)(z_{\alpha/2}^3+3z_{\alpha/2})
      +2k(1+k)\sigma_x^2\sigma_y^2 z_{\alpha/2}}
     {k\,(k\sigma_x^2+\sigma_y^2)^2}
\right\}.
\end{align*}
If the total sample size satisfies
\begin{equation}
\label{equ:Nmin2}
N \geq N_{\min}^{(2)} = \left(
  \frac{|a_1| + \sqrt{\,a_1^2 - 4|a_2|\epsilon \operatorname{sign}\bigl(a_1^2 - 4|a_2|\epsilon\bigr)}}
       {2\epsilon}
  \right)^{2},
\end{equation}
then the maximal tail deviation is controlled at the sharper level
\[
\max\{|\alpha_L^* - \alpha/2|,|\alpha_L^* - \alpha/2|\} \leq \epsilon + O(\epsilon^3).
\]
\end{theorem}
This theorem provides a closed-form criterion for determining 
the minimum samples size required for the $t$-test to achieve reliable inference.  In particular, $N_{\min}^{(1)}$ provides a simple first-order benchmark driven only by skewness, while $N_{\min}^{(2)}$ refines this benchmark by incorporating kurtosis and yielding sharper control. 
In practice, the values of skewness and kurtosis are unknown. One approach is to substitute the sample estimates of skewness and kurtosis directly into the formulas for $N_{\min}^{(1)}$ and $N_{\min}^{(2)}$. Alternatively, in online A/B testing platforms, historical experimental data can be leveraged to construct empirical priors for these cumulants, which can then be used to evaluate the thresholds more robustly.

Empirically we find that in most cases $N_{\min}^{(2)}$ is smaller than $N_{\min}^{(1)}$. This primarily reflects the fact that $a_1^2 \geq 4|a_2|\epsilon$ holds for small $\epsilon$ in many practical situations. As a result, the second order refinement typically reduces the required sample size, leading to more efficient testing without sacrificing error control.
Furthermore, for many online metrics, such as the number of diamonds received and the number of fans following a streamer, the skewness is extremely large, and the minimum required sample size can reach hundreds of millions.

    

\subsection{Edgeworth-Based Correction for Small Samples}
 
In many practical scenarios, however, the available sample size may fall below the theoretical thresholds established in Theorem~\ref{thm:threshold}.
Under such conditions, the standard $t$-test may no longer maintain the nominal Type~I error level, leading to unreliable or biased decisions.
To mitigate this issue, we introduce an Edgeworth-based correction for the $p$-value that enables valid hypothesis testing even in small-sample settings.

Let $G(x)$ denote the cumulative distribution function of the statistic $T$.
Following the classical Edgeworth expansion \citep{hall2013bootstrap}, we have
\begin{equation*}
    \label{equ:Gx}
G(x) \;=\; \Phi(x) \;+\; \phi(x)\{q_1(x)+q_2(x)\} \;+\; O(N^{-3/2}).
\end{equation*}
In our setting, the first- and second-order polynomials are
\begin{align}
\label{equ:q1q2}
q_1(x) &= \frac{\gamma_D}{6}\,(2x^2+1), \\[4pt]
q_2(x)
&= \frac{\tau_D - 3}{12}\,(x^3 - 3x)
   - \frac{\gamma_D^2}{18}\,(x^5 + 2x^3 - 3x) \notag \\
&\quad - \frac{1 + k}{4N}\,
\frac{(k^3\sigma_x^4 + \sigma_y^4)(x^3 + 3x) + 2k(1 + k)\sigma_x^2\sigma_y^2\,x}
     {k\,(k\sigma_x^2 + \sigma_y^2)^2}.
\end{align}
These expressions depend on the unknown skewness and kurtosis terms $\gamma_D$ and $\tau_D$.
Building upon this theoretical expansion, we introduce a plug-in correction for the $p$-value through the following steps.

First, we estimate sample skewness and kurtosis from the data,
$$\hat{\gamma}_x = (n_x\hat\sigma_x^{3})^{-1}\sum_{i=1}^{n_x} (X_i - \bar X)^3,\quad   
\hat{\tau}_x = (n_x\hat\sigma_x^{4})^{-1}\sum_{i=1}^{n_x} (X_i - \bar X)^4 $$
and obtain $\hat{\gamma}_y$ and $\hat{\tau}_y$ analogously. 
Substituting $\hat{\gamma}_x,\hat{\gamma}_y,\hat{\tau}_x,\hat{\tau}_y$ into the expressions for $\gamma_D$ and $\tau_D$ in Eq.~\eqref{equ:gammad} yields $\hat{\gamma}_D$ and $\hat{\tau}_D$.

Second, replacing $\gamma_D$ and $\tau_D$ with $\hat{\gamma}_D$ and $\hat{\tau}_D$ in Eq.~\eqref{equ:q1q2} provides the estimated correction terms $\hat{q}_1(x), \hat{q}_2(x)$, leading to the plug-in Edgeworth cumulative distribution function
\[
\hat G(x) = \Phi(x) + \phi(x)\{\hat q_1(x) + \hat q_2(x)\}.
\]
For numerical stability, we truncate $\hat G(x)$ to the unit interval as
$$
\hat G_{\mathrm{tr}}(x)
=\min\{\max\{\hat{G}(x), 0\}, 1\}.
$$
This truncation ensures that the estimated cumulative distribution function remains within valid probability bounds without affecting asymptotic properties.
The Edgeworth-corrected two-sided $p$-value is then computed as
\begin{equation}
    \label{equ:pc}
p_{c} = 2\min\{\hat G_{\mathrm{tr}}(T),1-\hat G_{\mathrm{tr}}(T)\}.
\end{equation}
In the final step, we perform hypothesis testing based on the corrected $p$-value.
The null hypothesis $H_0$ is rejected whenever $p_c < \alpha$.
When $T < 0$, the rejection occurs in the left tail, indicating that the control group outperforms the treatment.
Conversely, when $T > 0$, the rejection occurs in the right tail, suggesting that the treatment mean $\mu_y$ exceeds the control mean $\mu_x$.

By applying the corrected $p$-value, the test achieves more balanced and accurate control of Type~I error in both tails, even when the sample size is limited.

\section{Experiment}

\subsection{Synthetic Data Experiment}

\begin{table*}[t]
\centering
\caption{Deviation of Type I error rate at each tail and standard errors of the $t$-test ($p_t$) and Edgeworth-corrected test ($p_c$) across different sample sizes for \textit{lognormal} $LN(0,1)$ data with $k=5$, $\alpha = 0.05$ and $\epsilon = 0.01$ based on $B=10^4$ repeated experiments. }
\label{tab:tail_bias_se_synthlog1}
\setlength{\tabcolsep}{10pt}\renewcommand{\arraystretch}{1.25}
\begin{tabular}{l|cc|cc}
\toprule
\multirow{2}{*}{\textbf{Sample size $N$}}
  & \multicolumn{2}{|c|}{\textbf{The $t$-test}}
  & \multicolumn{2}{c}{\textbf{EE-corrected test}}\\
\cmidrule(lr){2-5}
   & $(\hat{\alpha}_L^*-\alpha/2)\ $
   & $(\hat{\alpha}_R^*-\alpha/2)\ $
   & $(\hat{\alpha}_L^*-\alpha/2)\ $
   & $(\hat{\alpha}_R^*-\alpha/2)\ $ \\
\midrule
1,500  & $-0.0131 \pm 0.0011$ & $0.0252 \pm 0.0022$ & $-0.0004 \pm 0.0015$ & $0.0142 \pm 0.0019$ \\
2,376  & $-0.0113 \pm 0.0012$ & $0.0171 \pm 0.0020$ & $0.0023 \pm 0.0016$  & $0.0083 \pm 0.0018$ \\
3,774  & $-0.0108 \pm 0.0012$ & $0.0135 \pm 0.0019$ & $-0.0005 \pm 0.0015$ & $0.0047 \pm 0.0017$ \\
\ulnum{5,988}  & \ulnum{$-0.0083 \pm 0.0013$} & \ulnum{$0.0091 \pm 0.0018$} & $0.0010 \pm 0.0016$  & $0.0010 \pm 0.0016$ \\
9,504  & $-0.0054 \pm 0.0014$ & $0.0106 \pm 0.0019$ & $0.0026 \pm 0.0016$  & $0.0044 \pm 0.0017$ \\
15,090 & $-0.0049 \pm 0.0014$ & $0.0063 \pm 0.0017$ & $0.0032 \pm 0.0017$  & $0.0012 \pm 0.0016$ \\
23,952 & $-0.0057 \pm 0.0014$ & $0.0076 \pm 0.0018$ & $0.0004 \pm 0.0016$  & $0.0036 \pm 0.0017$ \\
\bottomrule
\end{tabular}
\end{table*}

We begin by evaluating the finite-sample performance of Eqs.~\eqref{equ:Nmin1}, \eqref{equ:Nmin2} and the Edgeworth correction \eqref{equ:pc} through synthetic experiments. 
To this end, we generate samples $X_i$ and $Y_j$ independently from the lognormal distribution $LN(0,1)$, whose logarithm is standard normal.
This design ensures that the null hypothesis $H_0:\mu_x=\mu_y$ holds exactly, thereby providing a clean baseline for quantifying Type~I error behavior.  
Throughout the experiments, we fix the error tolerance at $\epsilon=0.01$. To study the impact of unequal allocation, we set the sample size ratio to $k=5$.
Substituting these values into \eqref{equ:Nmin1} and \eqref{equ:Nmin2} yields theoretical minimum sample size requirements of $N_{\min}^{(1)}=8712$ and $N_{\min}^{(2)}=5986$. These quantities serve as analytic benchmarks against which empirical performance can be assessed. 

To investigate how well these thresholds capture finite-sample behavior, we consider total sample sizes $N$ ranging from $N_{\min}^{(2)}/5$ to $5N_{\min}^{(2)}$,  selecting seven values evenly spaced on the logarithmic scale to cover both under- and over-powered regimes. For each $N$, we conduct $B=10^4$ independent replications. In replication $i=1,\cdots,B$, we compute the two sample $t$ statistic $T_i$ by Eq.~\eqref{equ:tstas} and the corresponding two-sided $p$-values for (i) the classical $t$-test, $p_{t,i}$ and (ii)  the Edgeworth-corrected $p_{c,i}$ by Eq.~\eqref{equ:pc} respectively. Let $I(\cdot)$ denote the indicator function. The empirical estimates of the left- and right-tail Type~I error are
$$\hat{\alpha}_{L}^{a} = B^{-1}\sum_{i=1}^B I(p_{a,i} < \alpha, T_i < 0), \quad \hat{\alpha}_R^{a} = B^{-1}\sum_{i=1}^B I(p_{a,i} < \alpha, T_i > 0),$$
where $a\in \{t,c\}$ indexes the $t$-test and the Edgeworth-corrected test, respectively. Deviations from the nominal per-tail level are 
$$\hat{\Delta}_L = |\hat{\alpha}_L^{a} - \alpha/2|,\qquad \hat{\Delta}_R = |\hat{\alpha}_R^{a} - \alpha/2|.$$

Table~\ref{tab:tail_bias_se_synthlog1} summarizes the findings. As expected, the tail deviations $\hat{\Delta}_L, \hat{\Delta}_R$ decrease as $N$ increases for both methods. Once $N \geq N_{\min}^{(2)}$, the deviations fall within the tolerance level $\epsilon$. When $N$ further exceeds the larger threshold $N_{\min}^{(1)}$, the deviations are also well controlled. These results confirm our theoretical analysis and illustrate that by incorporating higher-order information, kurtosis, of the data, one can compute a more precise minimum sample size $N_{\min}^{(2)}$. In contrast, relying only on skewness yields a valid but more conservative bound $N_{\min}^{(1)}$, which may require substantially larger samples. 
Notably, the corrected $p$-value $p_c$ achieves the tolerance at smaller $N$, highlighting its faster convergence to the nominal Type~I error rate compared to the two sample $t$-test.


\subsection{Real Data Experiment}

We next assess the proposed methods using real-world data from a large-scale online experimentation platform. Two representative kinds of data are considered, each exhibiting strong skewness and practical importance for platform operations. For both cases we set $\epsilon=0.01$ and apply two-sided tests with nominal significance level $\alpha=0.05$.

The first type of output data, \emph{publish count}, records the number of videos uploaded by users. This measure is central to platform engagement because a steady stream of newly published content sustains user interest and stimulates downstream outcomes such as views, shares, and interactions.
In practice, the distribution of publish count departs strongly from normality. A small fraction of highly active creators, typically only five to ten percent of the user base, contributes sixty to seventy percent of the total number of videos. This concentration produces substantial skewness and heavy tails, with sample estimates $\gamma_x = \gamma_y = 14.94$ and $\tau_x = \tau_y = 490.7$. Since the data have equal variances, these terms cancel out in Eqs.~\eqref{equ:Nmin1} and \eqref{equ:Nmin2}.
With a group size ratio of $k=5$, our theory provides two thresholds for the minimum required sample size. The first-order bound $N_{\min}^{(1)}$, which depends only on skewness, yields $N_{\min}^{(1)} = 51{,}094$. Incorporating both skewness and kurtosis gives the refined second-order bound $N_{\min}^{(2)} = 35{,}042$.
For this real-data experiment, we report the exact deviations of the empirical Type~I errors in each tail across different sample sizes, as summarized in Table~\ref{tab:tail_bias_se_pub}. The deviation of the classical $t$-test gradually decreases as $N$ increases and meets the tolerance threshold once $N$ approaches the theoretical bound $N_{\min}^{(2)} = 35{,}042$. 
In contrast, the Edgeworth-corrected $p$-value $p_c$ maintains $\hat{\Delta} \leq \epsilon$ consistently, even for substantially smaller sample sizes. This demonstrates the robustness of the introduced correction, which achieves accurate error control in regions where the standard $t$-test still exhibits notable bias.

\begin{table*}[t]
\centering
\caption{Deviation of Type I error rate at each tail and standard errors of $t$-test ($p_t$) and Edgeworth-corrected test ($p_c$) across different sample sizes for \textit{publish count} data with $k=5$, $\alpha = 0.05$ and $\epsilon = 0.01$ based on $B=10^4$ repeated experiments.}
\label{tab:tail_bias_se_pub}
\setlength{\tabcolsep}{10pt}\renewcommand{\arraystretch}{1.25}
\begin{tabular}{l|cc|cc}
\toprule
\multirow{2}{*}{\textbf{Sample size $N$}}
  & \multicolumn{2}{|c|}{\textbf{The $t$-test}}
  & \multicolumn{2}{c}{\textbf{EE-corrected test}}\\
\cmidrule(lr){2-5}
   & $(\hat{\alpha}_L^t-\alpha/2)\ $
   & $(\hat{\alpha}_R^t-\alpha/2)\ $
   & $(\hat{\alpha}_L^c-\alpha/2)\ $
   & $(\hat{\alpha}_R^c-\alpha/2)\ $ \\
\midrule
9,000   & $-0.0127 \pm 0.0011$ & $0.0245 \pm 0.0022$ & $-0.0003 \pm 0.0016$ & $0.0158 \pm 0.0020$ \\
12,000  & $-0.0121 \pm 0.0011$ & $0.0215 \pm 0.0021$ & $-0.0016 \pm 0.0015$ & $0.0127 \pm 0.0019$ \\
18,000  & $-0.0131 \pm 0.0011$ & $0.0173 \pm 0.0020$ & $0.0002 \pm 0.0016$  & $0.0092 \pm 0.0018$ \\
24,000  & $-0.0097 \pm 0.0012$ & $0.0125 \pm 0.0019$ & $0.0019 \pm 0.0016$  & $0.0054 \pm 0.0017$ \\
\ulnum{35,100}  & $\ulnum{-0.0079 \pm 0.0013}$ & $\ulnum{0.0103 \pm 0.0018}$ & $0.0046 \pm 0.0017$  & $0.0030 \pm 0.0016$ \\
48,000  & $-0.0063 \pm 0.0013$ & $0.0104 \pm 0.0018$ & $0.0052 \pm 0.0017$  & $0.0028 \pm 0.0016$ \\
72,000  & $-0.0057 \pm 0.0014$ & $0.0058 \pm 0.0017$ & $0.0032 \pm 0.0016$  & $-0.0013 \pm 0.0015$ \\
96,000  & $-0.0062 \pm 0.0014$ & $0.0056 \pm 0.0017$ & $0.0013 \pm 0.0016$  & $0.0007 \pm 0.0016$ \\
\bottomrule
\end{tabular}
\end{table*}


The second type of output data, \emph{live duration}, captures the length of live streaming sessions conducted by users. This variable plays a key role in platform dynamics, as the total amount of broadcasting shapes user engagement, advertising revenue, and the vitality of the online community. Longer sessions not only create more opportunities for interactions such as comments, likes, and virtual gifts but also help sustain audience participation throughout the live event.
Empirically, the distribution of live duration departs markedly from normality. On major social media platforms with integrated short-video and live-streaming services, most ordinary users post only videos and rarely host live sessions, so their live duration is recorded as zero.
In contrast, a much smaller group of professional hosts depend on live streaming as their main activity and often conduct sessions that last for many hours. This sharp contrast produces substantial skewness and heavy tails, with sample estimates $\gamma_x = \gamma_y = 5.09$ and $\tau_x = \tau_y = 41.9$.
For this experiment we set the allocation ratio to $k=10$. The theoretical framework yields two benchmarks for the minimum required sample size. The first order bound is $N_{\min}^{(1)} = 15{,}022$, while the refined second order bound is $N_{\min}^{(2)} = 9{,}361$.

As shown in Table~\ref{tab:tail_bias_se_live}, the empirical results for live duration exhibit the same qualitative pattern as those for publish count.
Once the total sample size reaches the theoretical threshold $N_{\min}^{(2)}$, the Type~I error of the $t$-test aligns with the nominal level. 
In the smaller-sample regime, however, the $t$-test still exhibits noticeable deviation in Type~I error, especially in the right tail. By contrast, the Edgeworth-corrected $p$-value $p_c$ effectively stabilizes the error across all sample sizes. Even for $N$ well below the theoretical threshold, the deviations $|\hat{\alpha}_L^c - \alpha/2|$ and $|\hat{\alpha}_R^c - \alpha/2|$ remain within the tolerance $\epsilon = 0.01$.
These numerical results confirm that the Edgeworth correction generalizes beyond a single data type and remains reliable in highly skewed and unbalanced cases such as live streaming duration.

\begin{table*}[t]
\centering 
\caption{Deviation of Type I error rate at each tail and standard errors of $t$-test ($p_t$) and Edgeworth-corrected test ($p_c$) across different sample sizes for \textit{live duration} data with $k=10$, $\alpha = 0.05$ and $\epsilon = 0.01$ based on $B=10^4$ repeated experiments.}
\label{tab:tail_bias_se_live}
\setlength{\tabcolsep}{10pt}\renewcommand{\arraystretch}{1.25}
\begin{tabular}{l|cc|cc}
\toprule
\multirow{2}{*}{\textbf{Sample size $N$}}
  & \multicolumn{2}{|c|}{\textbf{The $t$-test}}
  & \multicolumn{2}{c}{\textbf{EE-corrected test}}\\
\cmidrule(lr){2-5}
   & $(\hat{\alpha}_L^t-\alpha/2)\ $
   & $(\hat{\alpha}_R^t-\alpha/2)\ $
   & $(\hat{\alpha}_L^c-\alpha/2)\ $
   & $(\hat{\alpha}_R^c-\alpha/2)\ $ \\
\midrule
2,200   & $-0.0151 \pm 0.0010$ & $0.0293 \pm 0.0023$ & $-0.0033 \pm 0.0015$  & $0.0102 \pm 0.0018$ \\
3,300   & $-0.0146 \pm 0.0010$ & $0.0242 \pm 0.0022$ & $-0.0019 \pm 0.0015$  & $0.0068 \pm 0.0018$ \\
4,950   & $-0.0102 \pm 0.0012$ & $0.0164 \pm 0.0020$ & $0.0005 \pm 0.0016$   & $0.0019 \pm 0.0016$ \\
6,600   & $-0.0107 \pm 0.0012$ & $0.0156 \pm 0.0020$ & $0.0004 \pm 0.0016$   & $0.0000 \pm 0.0016$ \\
\ulnum{9,350}   & $\ulnum{-0.0080 \pm 0.0013}$ & $\ulnum{0.0094 \pm 0.0018}$ & $0.0033 \pm 0.0017$   & $-0.0030 \pm 0.0015$ \\
13,200  & $-0.0082 \pm 0.0013$ & $0.0074 \pm 0.0018$ & $0.0014 \pm 0.0016$   & $-0.0018 \pm 0.0015$ \\
18,700  & $-0.0049 \pm 0.0014$ & $0.0049 \pm 0.0017$ & $0.0029 \pm 0.0016$   & $-0.0033 \pm 0.0015$ \\
27,500  & $-0.0054 \pm 0.0014$ & $0.0033 \pm 0.0017$ & $0.0013 \pm 0.0016$   & $-0.0039 \pm 0.0014$ \\
\bottomrule
\end{tabular}
\end{table*}



Across both data types, the empirical findings reinforce our theoretical results. The minimum required sample sizes derived from theory closely match those observed in practice, confirming the accuracy of the second order bound $N_{\min}^{(2)}$. In most cases $N_{\min}^{(2)}$ is not only more precise but also smaller than the first order bound $N_{\min}^{(1)}$, providing a sharper threshold that reduces unnecessary conservatism. For large scale A/B testing platforms, even a modest reduction in sample size can translate into substantial savings in experimental resources. Empirical estimates suggest that applying $N_{\min}^{(2)}$ in practice could save on the order of one hundred million dollars in experimentation costs, underscoring the substantial practical value of our methodology for large-scale online platforms.
At the same time, in small sample scenarios where higher order moment estimates such as kurtosis may be unstable, the first order bound $N_{\min}^{(1)}$ remains valuable as a conservative guideline. For this reason we retain both thresholds, so that practitioners can flexibly choose between a more precise or a more robust benchmark depending on the data available.

\subsection{In-Depth Analysis}

In this section, we examine the robustness of our method under varying settings of the tolerance level $\epsilon$ and the allocation ratio $k$.

First, we examine how the minimal required sample sizes, $N_{\min}^{(1)}$ and $N_{\min}^{(2)}$, vary with the tolerance level $\epsilon$. We employ the two real-world datasets introduced in the previous section, namely the publish count and live duration datasets. The significance level is fixed at $\alpha = 0.1$.
We adopt $\alpha = 0.1$ instead of $\alpha = 0.05$ to ensure a feasible range of $\epsilon$. When $\alpha$ is smaller, the admissible values of $\epsilon$ become very limited and may fall below the Monte Carlo error (approximately $0.003$ under $10{,}000$ replications), leading to unstable numerical estimates. Hence, $\alpha = 0.1$ offers a more reliable setting for investigating the sensitivity of $N_{\min}$ with respect to $\epsilon$.

The error tolerance $\epsilon$ is varied over $\{0.04, 0.03, 0.02, 0.01\}$, while all other parameters remain identical to those used in the preceding experiments. For each $\epsilon$, our theoretical formulas Eqs.~\eqref{equ:Nmin1} and \eqref{equ:Nmin2} yield the corresponding $N_{\min}^{(1)}$ and $N_{\min}^{(2)}$. To empirically validate these results, we perform $10^4$ Monte Carlo replications and identify the smallest $N$ satisfying
$$\max\{|\hat \alpha^*_R-\alpha/2|,\,|\hat \alpha^*_L-\alpha/2|\}\leq \epsilon.$$

As shown in Figure~\ref{fig:results-abl}, both $N_{\min}^{(1)}$, $N_{\min}^{(2)}$, and the empirically estimated minimum sample size increase as $\epsilon$ decreases. Although the theoretical values do not perfectly coincide with the empirical estimates, they remain sufficiently close to provide accurate guidance on the required sample size for achieving the desired Type-I error control in general A/B testing scenarios. This consistency demonstrates the robustness of our theoretical derivation with respect to $\epsilon$ and highlights its practical reliability for real-world experimentation.

\begin{figure}[t]
    \includegraphics[width=0.49\linewidth]{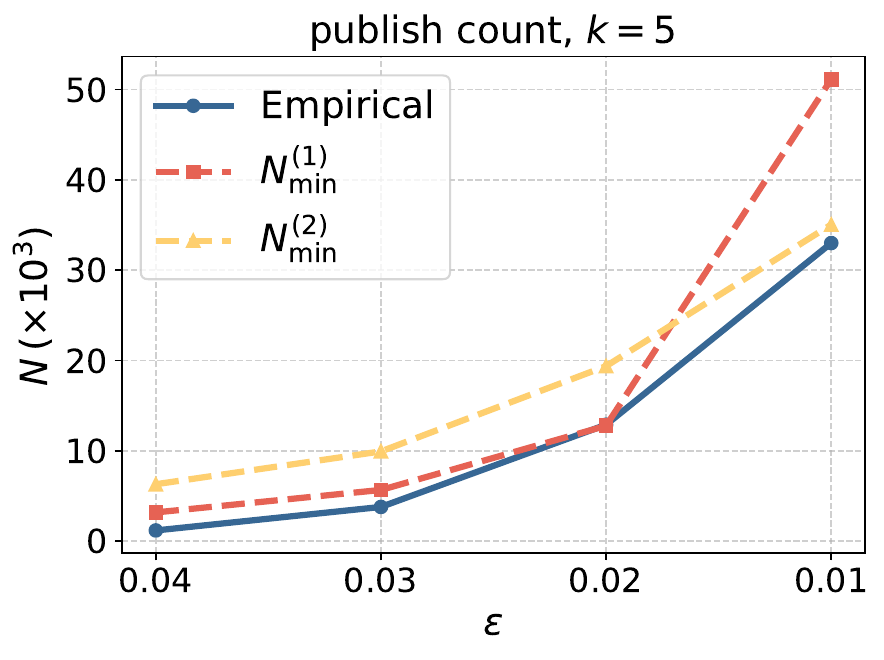}
    \hfill
    \includegraphics[width=0.49\linewidth]{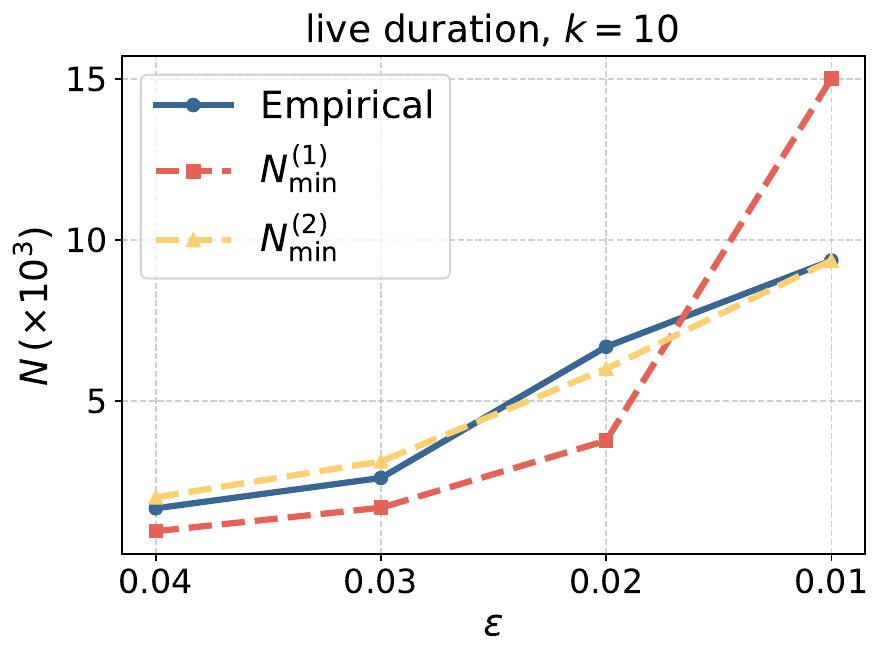}
    \caption{Empirical vs. theoretical minimal sample sizes $N_{\min}^{(1)}$ and $N_{\min}^{(2)}$ under $\alpha=0.1$ across varying tolerance levels $\epsilon$ for (left) publish count and (right) live duration datasets.}
    \label{fig:results-abl}
\end{figure}


\begin{figure}[t]
  \centering
  \begin{subfigure}[t]{0.49\linewidth}
    \includegraphics[width=\linewidth]{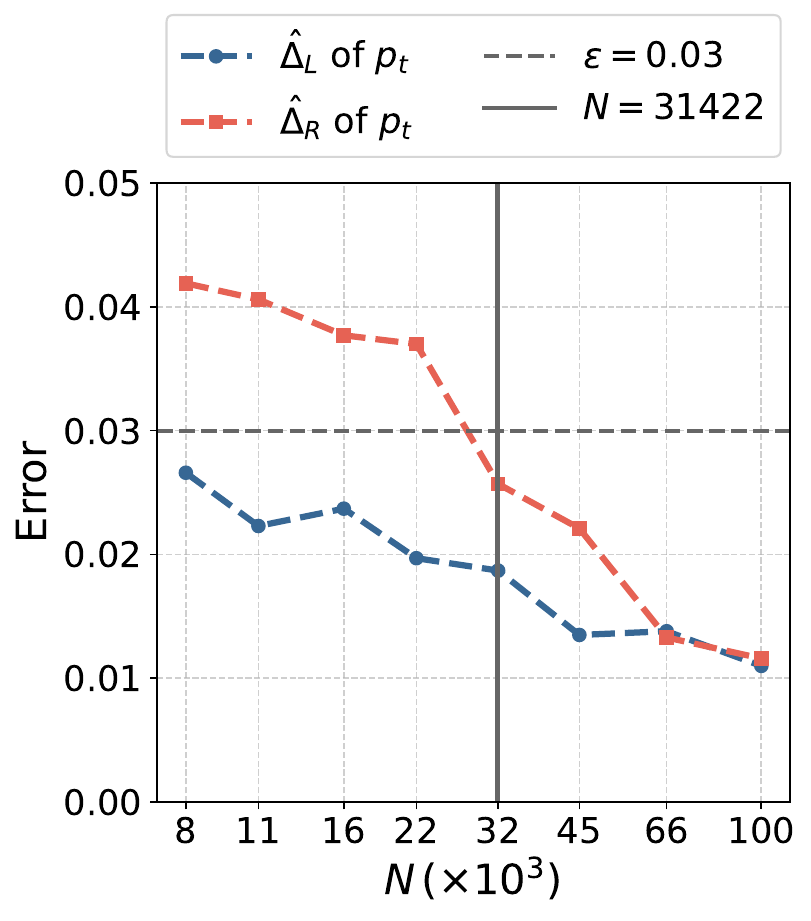}
    \caption{Publish count, $k=9$ for $p_t$.}
  \end{subfigure}\hfill
  \begin{subfigure}[t]{0.49\linewidth}
    \includegraphics[width=\linewidth]{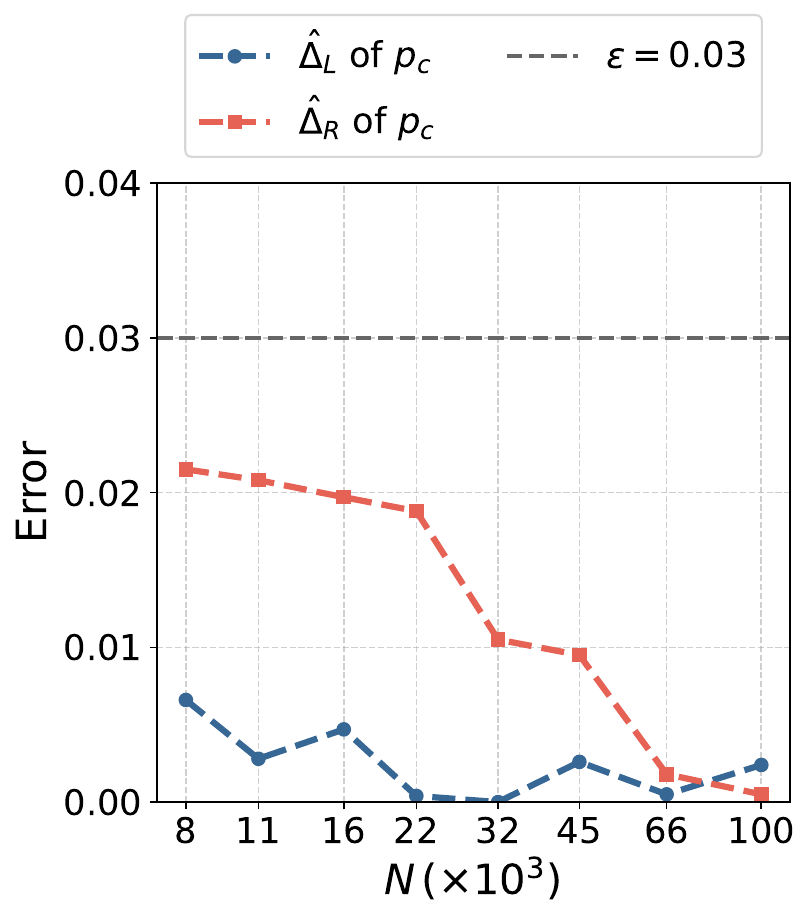}
    \caption{Publish count, $k=9$ for $p_c$.}
  \end{subfigure}

  \vspace{0.4em}
  \begin{subfigure}[t]{0.49\linewidth}
    \includegraphics[width=\linewidth]{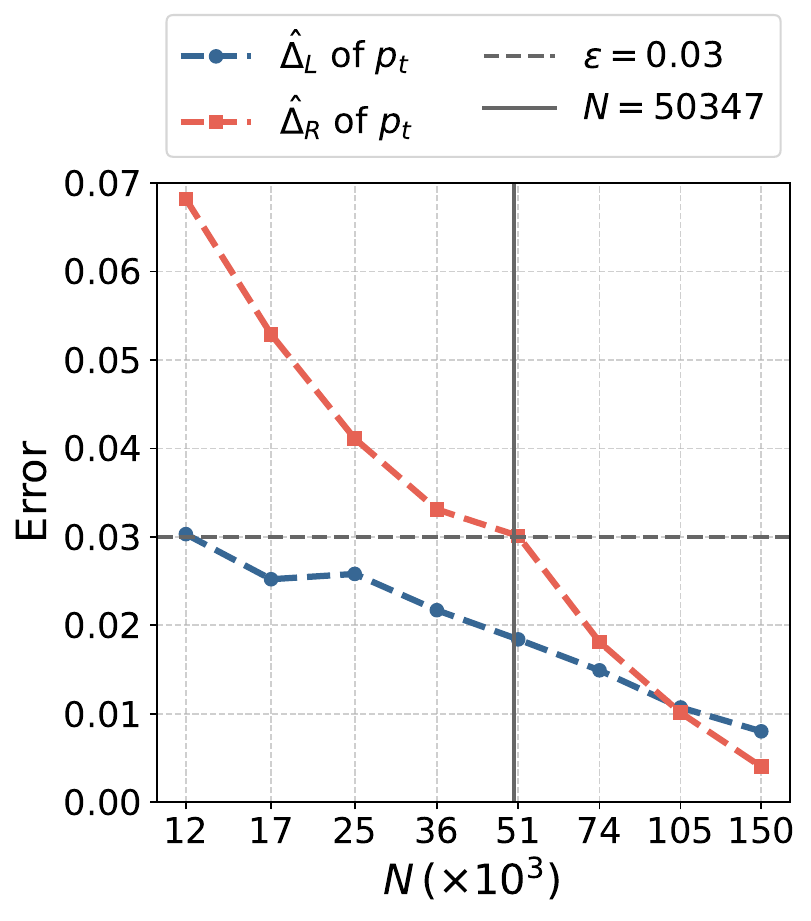}
    \caption{Live duration, $k=99$ for $p_t$.}
  \end{subfigure}\hfill
  \begin{subfigure}[t]{0.49\linewidth}
    \includegraphics[width=\linewidth]{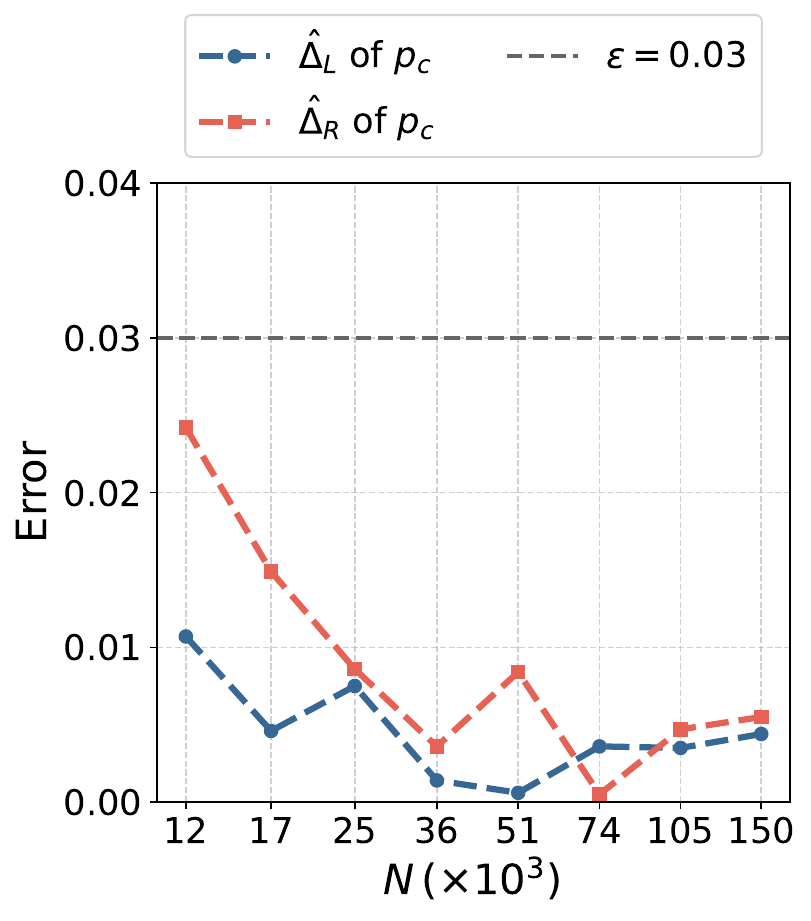}
    \caption{Live duration, $k=99$ for $p_c$.}
  \end{subfigure}

  \caption{In-depth comparison at $\alpha=0.10$ and $\epsilon=0.03$ for two metrics (publish count, live duration) under two statistics (Welch $t$ vs.\ Edgeworth-corrected $p_c$).}
  \label{fig:results-n-tail}
\end{figure}

Secondly, we investigate the effect of varying the significance level $\alpha$ and the allocation ratio $k$. In practice, online platforms often adopt a relatively large $\alpha$ to increase sensitivity in detecting potentially beneficial strategies. Here, we set $\alpha = 0.1$, and since a larger significance level allows for a higher tolerance, we take $\epsilon = 0.03$.
Moreover, we examine highly unbalanced allocation scenarios. Specifically, we take $k = 9$ for the publish count dataset and $k = 99$ for the live duration dataset. Such extreme ratios commonly arise in the early stages of experimentation, where a small portion of users is assigned to a new strategy for preliminary evaluation. In real-world A/B testing platforms, these unbalanced settings are often constrained by business or policy considerations.

For the publish count case with $k = 9$, our theoretical result in Eq.~\eqref{equ:Nmin2} gives a minimum required sample size of $N_{\min}^{(2)} = 31{,}422$. As shown in Figure~\ref{fig:results-n-tail}, when $N \geq 31{,}500$, the empirical tail errors fall below the threshold $\epsilon = 0.03$, which is very close to our theoretical prediction.
Similarly, for the live duration case with $k = 99$, representing a highly extreme allocation, the empirical minimum sample size reported in Figure~\ref{fig:results-n-tail} is $N = 51{,}300$ to satisfy the same error criterion, while our theoretical estimate gives $N_{\min}^{(2)} = 50{,}347$. The close agreement between theory and simulation again confirms the reliability of our analytical formulation under extreme allocation scenarios.

\section{Conclusion}

This work investigates and quantifies how skewed, long-tailed data and unequal allocation jointly affect the deviation of Type~I error in two-sample $t$-tests.
Building on this analysis, we derive explicit conditions and a minimum sample size threshold that ensure reliable inference under general asymmetry and imbalance. 
For small-sample scenarios where these conditions are not met, we introduce an Edgeworth-based correction that effectively improves Type~I error control. 
Both theoretical analysis and large-scale empirical evaluations confirm the validity and robustness of our approach. 
The refined threshold $N_{\min}^{(2)}$ aligns closely with practical observations while substantially reducing conservatism, leading to meaningful savings in traffic and computational resources for large-scale online platforms. 
Future research will focus on extending this framework to multi-arm and sequential testing settings, where similar asymmetry-induced effects are expected to arise.

\balance
\bibliographystyle{ACM-Reference-Format}
\bibliography{skew}


\begin{thebibliography}{24}


\ifx \showCODEN    \undefined \def \showCODEN     #1{\unskip}     \fi
\ifx \showISBNx    \undefined \def \showISBNx     #1{\unskip}     \fi
\ifx \showISBNxiii \undefined \def \showISBNxiii  #1{\unskip}     \fi
\ifx \showISSN     \undefined \def \showISSN      #1{\unskip}     \fi
\ifx \showLCCN     \undefined \def \showLCCN      #1{\unskip}     \fi
\ifx \shownote     \undefined \def \shownote      #1{#1}          \fi
\ifx \showarticletitle \undefined \def \showarticletitle #1{#1}   \fi
\ifx \showURL      \undefined \def \showURL       {\relax}        \fi
\providecommand\bibfield[2]{#2}
\providecommand\bibinfo[2]{#2}
\providecommand\natexlab[1]{#1}
\providecommand\showeprint[2][]{arXiv:#2}

\bibitem[Boos and Hughes-Oliver(2000)]%
        {Boos}
\bibfield{author}{\bibinfo{person}{Dennis~D. Boos} {and} \bibinfo{person}{Jacqueline~M. Hughes-Oliver}.} \bibinfo{year}{2000}\natexlab{}.
\newblock \showarticletitle{How large does $n$ have to be for $Z$ and $t$ intervals?}
\newblock \bibinfo{journal}{\emph{The American Statistician}} \bibinfo{volume}{54}, \bibinfo{number}{2} (\bibinfo{year}{2000}), \bibinfo{pages}{121--128}.
\newblock
\href{https://doi.org/10.1080/00031305.2000.10474524}{doi:\nolinkurl{10.1080/00031305.2000.10474524}}


\bibitem[Chen(1995)]%
        {chenlin1995}
\bibfield{author}{\bibinfo{person}{Ling Chen}.} \bibinfo{year}{1995}\natexlab{}.
\newblock \showarticletitle{Testing the mean of skewed distributions}.
\newblock \bibinfo{journal}{\emph{J. Amer. Statist. Assoc.}} \bibinfo{volume}{90}, \bibinfo{number}{430} (\bibinfo{year}{1995}), \bibinfo{pages}{767--772}.
\newblock
\showISSN{01621459, 1537274X}
\urldef\tempurl%
\url{http://www.jstor.org/stable/2291090}
\showURL{%
\tempurl}


\bibitem[Deng and Shi(2016)]%
        {deng16micro}
\bibfield{author}{\bibinfo{person}{Alex Deng} {and} \bibinfo{person}{Xiaolin Shi}.} \bibinfo{year}{2016}\natexlab{}.
\newblock \showarticletitle{Data-driven metric development for online controlled experiments: Seven lessons learned}. In \bibinfo{booktitle}{\emph{Proceedings of the 22nd ACM SIGKDD International Conference on Knowledge Discovery and Data Mining}} (San Francisco, California, USA) \emph{(\bibinfo{series}{KDD '16})}. \bibinfo{publisher}{Association for Computing Machinery}, \bibinfo{address}{New York, NY, USA}, \bibinfo{pages}{77–86}.
\newblock
\showISBNx{9781450342322}
\href{https://doi.org/10.1145/2939672.2939700}{doi:\nolinkurl{10.1145/2939672.2939700}}


\bibitem[Deng et~al\mbox{.}(2013)]%
        {deng2013improving}
\bibfield{author}{\bibinfo{person}{Alex Deng}, \bibinfo{person}{Ya Xu}, \bibinfo{person}{Ron Kohavi}, {and} \bibinfo{person}{Toby Walker}.} \bibinfo{year}{2013}\natexlab{}.
\newblock \showarticletitle{Improving the sensitivity of online controlled experiments by utilizing pre-experiment data}. In \bibinfo{booktitle}{\emph{Proceedings of the sixth ACM international conference on Web search and data mining}}. \bibinfo{pages}{123--132}.
\newblock


\bibitem[Hall(2013)]%
        {hall2013bootstrap}
\bibfield{author}{\bibinfo{person}{Peter Hall}.} \bibinfo{year}{2013}\natexlab{}.
\newblock \bibinfo{booktitle}{\emph{The bootstrap and {Edgeworth} expansion}}.
\newblock \bibinfo{publisher}{Springer Science \& Business Media}.
\newblock


\bibitem[Hall and Wang(2004)]%
        {hall2004exact}
\bibfield{author}{\bibinfo{person}{Peter Hall} {and} \bibinfo{person}{Qiying Wang}.} \bibinfo{year}{2004}\natexlab{}.
\newblock \showarticletitle{Exact convergence rate and leading term in central limit theorem for {Student’s} t statistic}.
\newblock \bibinfo{journal}{\emph{The Annals of Probability}} \bibinfo{volume}{32}, \bibinfo{number}{2} (\bibinfo{date}{April} \bibinfo{year}{2004}).
\newblock
\showISSN{0091-1798}
\href{https://doi.org/10.1214/009117904000000252}{doi:\nolinkurl{10.1214/009117904000000252}}


\bibitem[Han et~al\mbox{.}(2023)]%
        {increase23}
\bibfield{author}{\bibinfo{person}{Kevin Han}, \bibinfo{person}{Shuangning Li}, \bibinfo{person}{Jialiang Mao}, {and} \bibinfo{person}{Han Wu}.} \bibinfo{year}{2023}\natexlab{}.
\newblock \showarticletitle{Detecting interference in online controlled experiments with increasing allocation}. In \bibinfo{booktitle}{\emph{Proceedings of the 29th ACM SIGKDD Conference on Knowledge Discovery and Data Mining}} (Long Beach, CA, USA) \emph{(\bibinfo{series}{KDD '23})}. \bibinfo{publisher}{Association for Computing Machinery}, \bibinfo{address}{New York, NY, USA}, \bibinfo{pages}{661–672}.
\newblock
\showISBNx{9798400701030}
\href{https://doi.org/10.1145/3580305.3599308}{doi:\nolinkurl{10.1145/3580305.3599308}}


\bibitem[Johari et~al\mbox{.}(2017)]%
        {Johari17}
\bibfield{author}{\bibinfo{person}{Ramesh Johari}, \bibinfo{person}{Pete Koomen}, \bibinfo{person}{Leonid Pekelis}, {and} \bibinfo{person}{David Walsh}.} \bibinfo{year}{2017}\natexlab{}.
\newblock \showarticletitle{Peeking at A/B Tests: Why it matters, and what to do about it}. In \bibinfo{booktitle}{\emph{Proceedings of the 23rd ACM SIGKDD International Conference on Knowledge Discovery and Data Mining}} (Halifax, NS, Canada) \emph{(\bibinfo{series}{KDD '17})}. \bibinfo{publisher}{Association for Computing Machinery}, \bibinfo{address}{New York, NY, USA}, \bibinfo{pages}{1517–1525}.
\newblock
\showISBNx{9781450348874}
\href{https://doi.org/10.1145/3097983.3097992}{doi:\nolinkurl{10.1145/3097983.3097992}}


\bibitem[Johari et~al\mbox{.}(2022)]%
        {or22}
\bibfield{author}{\bibinfo{person}{Ramesh Johari}, \bibinfo{person}{Pete Koomen}, \bibinfo{person}{Leonid Pekelis}, {and} \bibinfo{person}{David Walsh}.} \bibinfo{year}{2022}\natexlab{}.
\newblock \showarticletitle{Always valid inference: Continuous monitoring of A/B tests}.
\newblock \bibinfo{journal}{\emph{Operations Research}} \bibinfo{volume}{70}, \bibinfo{number}{3} (\bibinfo{year}{2022}), \bibinfo{pages}{1806--1821}.
\newblock


\bibitem[Johnson(1978)]%
        {Johnson1978}
\bibfield{author}{\bibinfo{person}{Norman~J. Johnson}.} \bibinfo{year}{1978}\natexlab{}.
\newblock \showarticletitle{Modified $t$ tests and confidence intervals for asymmetrical populations}.
\newblock \bibinfo{journal}{\emph{J. Amer. Statist. Assoc.}} \bibinfo{volume}{73}, \bibinfo{number}{363} (\bibinfo{year}{1978}), \bibinfo{pages}{536--544}.
\newblock
\showISSN{01621459, 1537274X}
\urldef\tempurl%
\url{http://www.jstor.org/stable/2286597}
\showURL{%
\tempurl}


\bibitem[Kohavi and Longbotham(2011)]%
        {Kohavi10unexpected}
\bibfield{author}{\bibinfo{person}{Ron Kohavi} {and} \bibinfo{person}{Roger Longbotham}.} \bibinfo{year}{2011}\natexlab{}.
\newblock \showarticletitle{Unexpected results in online controlled experiments}.
\newblock \bibinfo{journal}{\emph{SIGKDD Explor. Newsl.}} \bibinfo{volume}{12}, \bibinfo{number}{2} (\bibinfo{date}{March} \bibinfo{year}{2011}), \bibinfo{pages}{31–35}.
\newblock
\showISSN{1931-0145}
\href{https://doi.org/10.1145/1964897.1964905}{doi:\nolinkurl{10.1145/1964897.1964905}}


\bibitem[Kohavi et~al\mbox{.}(2020)]%
        {kohavi2020trustworthy}
\bibfield{author}{\bibinfo{person}{Ron Kohavi}, \bibinfo{person}{Diane Tang}, {and} \bibinfo{person}{Ya Xu}.} \bibinfo{year}{2020}\natexlab{}.
\newblock \bibinfo{booktitle}{\emph{Trustworthy online controlled experiments: A practical guide to {A/B} testing}}.
\newblock \bibinfo{publisher}{Cambridge University Press}.
\newblock
\showISBNx{9781108724265}
\href{https://doi.org/10.1017/9781108653985}{doi:\nolinkurl{10.1017/9781108653985}}


\bibitem[Larsen et~al\mbox{.}(2024)]%
        {larsen2024statistical}
\bibfield{author}{\bibinfo{person}{Nicholas Larsen}, \bibinfo{person}{Jonathan Stallrich}, \bibinfo{person}{Srijan Sengupta}, \bibinfo{person}{Alex Deng}, \bibinfo{person}{Ron Kohavi}, {and} \bibinfo{person}{Nathaniel~T. Stevens}.} \bibinfo{year}{2024}\natexlab{}.
\newblock \showarticletitle{Statistical challenges in online controlled experiments: A review of {A/B} testing methodology}.
\newblock \bibinfo{journal}{\emph{The American Statistician}} \bibinfo{volume}{78}, \bibinfo{number}{2} (\bibinfo{year}{2024}), \bibinfo{pages}{135--149}.
\newblock


\bibitem[Lehmann and Romano(2005)]%
        {lehmann2005pvalue}
\bibfield{author}{\bibinfo{person}{Erich~Leo Lehmann} {and} \bibinfo{person}{Joseph~P Romano}.} \bibinfo{year}{2005}\natexlab{}.
\newblock \bibinfo{booktitle}{\emph{Testing statistical hypotheses}}.
\newblock \bibinfo{publisher}{Springer}.
\newblock


\bibitem[Liu et~al\mbox{.}(2021)]%
        {liu21link}
\bibfield{author}{\bibinfo{person}{Min Liu}, \bibinfo{person}{Jialiang Mao}, {and} \bibinfo{person}{Kang Kang}.} \bibinfo{year}{2021}\natexlab{}.
\newblock \showarticletitle{Trustworthy and powerful online marketplace experimentation with budget-split design}. In \bibinfo{booktitle}{\emph{Proceedings of the 27th ACM SIGKDD Conference on Knowledge Discovery \& Data Mining}} (Virtual Event, Singapore) \emph{(\bibinfo{series}{KDD '21})}. \bibinfo{publisher}{Association for Computing Machinery}, \bibinfo{address}{New York, NY, USA}, \bibinfo{pages}{3319–3329}.
\newblock
\showISBNx{9781450383325}
\href{https://doi.org/10.1145/3447548.3467193}{doi:\nolinkurl{10.1145/3447548.3467193}}


\bibitem[McMahan et~al\mbox{.}(2013)]%
        {2013Adclick}
\bibfield{author}{\bibinfo{person}{H. McMahan}, \bibinfo{person}{Gary Holt}, \bibinfo{person}{D. Sculley}, \bibinfo{person}{Michael Young}, \bibinfo{person}{Dietmar Ebner}, \bibinfo{person}{Julian Grady}, \bibinfo{person}{Lan Nie}, \bibinfo{person}{Todd Phillips}, \bibinfo{person}{Eugene Davydov}, \bibinfo{person}{Daniel Golovin}, \bibinfo{person}{Sharat Chikkerur}, \bibinfo{person}{Dan Liu}, \bibinfo{person}{Martin Wattenberg}, \bibinfo{person}{Arnar Hrafnkelsson}, \bibinfo{person}{Tom Boulos}, {and} \bibinfo{person}{Jeremy Kubica}.} \bibinfo{year}{2013}\natexlab{}.
\newblock \showarticletitle{Ad click prediction: A view from the trenches}. In \bibinfo{booktitle}{\emph{Proceedings of the 19th ACM SIGKDD International Conference on Knowledge Discovery and Data Mining (KDD '13)}}. \bibinfo{publisher}{Association for Computing Machinery}, \bibinfo{address}{New York, NY, USA}, \bibinfo{pages}{1222--1230}.
\newblock
\href{https://doi.org/10.1145/2487575.2488200}{doi:\nolinkurl{10.1145/2487575.2488200}}


\bibitem[Oldfield(2016)]%
        {16unequal}
\bibfield{author}{\bibinfo{person}{Marie Oldfield}.} \bibinfo{year}{2016}\natexlab{}.
\newblock \showarticletitle{Unequal sample sizes and the use of larger control groups pertaining to power of a study -Published by Ministry of Defence UK Paper : DSTLTR92592 P2PP2R-2016-02-23T13}.
\newblock  (\bibinfo{date}{04} \bibinfo{year}{2016}).
\newblock
\href{https://doi.org/10.6084/m9.figshare.23988477}{doi:\nolinkurl{10.6084/m9.figshare.23988477}}


\bibitem[Sutton(1993)]%
        {Sutton1993}
\bibfield{author}{\bibinfo{person}{Clifton~D. Sutton}.} \bibinfo{year}{1993}\natexlab{}.
\newblock \showarticletitle{Computer-intensive methods for tests about the mean of an asymmetrical distribution}.
\newblock \bibinfo{journal}{\emph{J. Amer. Statist. Assoc.}} \bibinfo{volume}{88}, \bibinfo{number}{423} (\bibinfo{year}{1993}), \bibinfo{pages}{802--810}.
\newblock
\showeprint{https://www.tandfonline.com/doi/pdf/10.1080/01621459.1993.10476345}
\href{https://doi.org/10.1080/01621459.1993.10476345}{doi:\nolinkurl{10.1080/01621459.1993.10476345}}


\bibitem[Urbano et~al\mbox{.}(2021)]%
        {metric21typeI}
\bibfield{author}{\bibinfo{person}{Juli\'{a}n Urbano}, \bibinfo{person}{Matteo Corsi}, {and} \bibinfo{person}{Alan Hanjalic}.} \bibinfo{year}{2021}\natexlab{}.
\newblock \showarticletitle{How do metric score distributions affect the Type I error rate of statistical significance tests in information retrieval?}. In \bibinfo{booktitle}{\emph{Proceedings of the 2021 ACM SIGIR International Conference on Theory of Information Retrieval}} (Virtual Event, Canada) \emph{(\bibinfo{series}{ICTIR '21})}. \bibinfo{publisher}{Association for Computing Machinery}, \bibinfo{address}{New York, NY, USA}, \bibinfo{pages}{245–250}.
\newblock
\showISBNx{9781450386111}
\href{https://doi.org/10.1145/3471158.3472242}{doi:\nolinkurl{10.1145/3471158.3472242}}


\bibitem[Welch(1947)]%
        {welch}
\bibfield{author}{\bibinfo{person}{B.~L. Welch}.} \bibinfo{year}{1947}\natexlab{}.
\newblock \showarticletitle{The generalization of `Student's' problem when several different population variances are involved}.
\newblock \bibinfo{journal}{\emph{Biometrika}} \bibinfo{volume}{34}, \bibinfo{number}{1/2} (\bibinfo{year}{1947}), \bibinfo{pages}{28--35}.
\newblock
\showISSN{00063444}
\urldef\tempurl%
\url{http://www.jstor.org/stable/2332510}
\showURL{%
\tempurl}


\bibitem[Xie and Aurisset(2016)]%
        {XieAurisset2016KDD}
\bibfield{author}{\bibinfo{person}{Huizhi Xie} {and} \bibinfo{person}{Juliette Aurisset}.} \bibinfo{year}{2016}\natexlab{}.
\newblock \showarticletitle{Improving the sensitivity of online controlled experiments: Case studies at Netflix}. In \bibinfo{booktitle}{\emph{Proceedings of the 22nd ACM SIGKDD International Conference on Knowledge Discovery and Data Mining}} (San Francisco, California, USA) \emph{(\bibinfo{series}{KDD '16})}. \bibinfo{publisher}{Association for Computing Machinery}, \bibinfo{address}{New York, NY, USA}, \bibinfo{pages}{645–654}.
\newblock
\showISBNx{9781450342322}
\href{https://doi.org/10.1145/2939672.2939733}{doi:\nolinkurl{10.1145/2939672.2939733}}


\bibitem[Xu et~al\mbox{.}(2015)]%
        {xu15link}
\bibfield{author}{\bibinfo{person}{Ya Xu}, \bibinfo{person}{Nanyu Chen}, \bibinfo{person}{Addrian Fernandez}, \bibinfo{person}{Omar Sinno}, {and} \bibinfo{person}{Anmol Bhasin}.} \bibinfo{year}{2015}\natexlab{}.
\newblock \showarticletitle{From infrastructure to culture: A/B testing challenges in large-scale social networks}. In \bibinfo{booktitle}{\emph{Proceedings of the 21th ACM SIGKDD International Conference on Knowledge Discovery and Data Mining}} (Sydney, NSW, Australia) \emph{(\bibinfo{series}{KDD '15})}. \bibinfo{publisher}{Association for Computing Machinery}, \bibinfo{address}{New York, NY, USA}, \bibinfo{pages}{2227–2236}.
\newblock
\showISBNx{9781450336642}
\href{https://doi.org/10.1145/2783258.2788602}{doi:\nolinkurl{10.1145/2783258.2788602}}


\bibitem[Xu et~al\mbox{.}(2018)]%
        {2018SQRLink}
\bibfield{author}{\bibinfo{person}{Ya Xu}, \bibinfo{person}{Weitao Duan}, {and} \bibinfo{person}{Shaochen Huang}.} \bibinfo{year}{2018}\natexlab{}.
\newblock \showarticletitle{SQR: Balancing speed, quality and risk in online experiments}. In \bibinfo{booktitle}{\emph{Proceedings of the 24th ACM SIGKDD International Conference on Knowledge Discovery \& Data Mining}} (London, United Kingdom) \emph{(\bibinfo{series}{KDD '18})}. \bibinfo{publisher}{Association for Computing Machinery}, \bibinfo{address}{New York, NY, USA}, \bibinfo{pages}{895–904}.
\newblock
\showISBNx{9781450355520}
\href{https://doi.org/10.1145/3219819.3219875}{doi:\nolinkurl{10.1145/3219819.3219875}}


\bibitem[Zhou et~al\mbox{.}(2023)]%
        {samplesize23}
\bibfield{author}{\bibinfo{person}{Jing Zhou}, \bibinfo{person}{Jiannan Lu}, {and} \bibinfo{person}{Anas Shallah}.} \bibinfo{year}{2023}\natexlab{}.
\newblock \showarticletitle{All about sample-Size calculations for A/B testing: novel extensions \& practical guide}. In \bibinfo{booktitle}{\emph{Proceedings of the 32nd ACM International Conference on Information and Knowledge Management}} (Birmingham, United Kingdom) \emph{(\bibinfo{series}{CIKM '23})}. \bibinfo{publisher}{Association for Computing Machinery}, \bibinfo{address}{New York, NY, USA}, \bibinfo{pages}{3574–3583}.
\newblock
\showISBNx{9798400701245}
\href{https://doi.org/10.1145/3583780.3614779}{doi:\nolinkurl{10.1145/3583780.3614779}}


\end{thebibliography}

\newpage
\appendix
\section{Proofs}

\begin{proof}[Proof of Theorem~\ref{thm:cumulants}]
Define $A=\bar Y-\mu_y$ and $B=\bar X-\mu_x$, so that
$D = A - B$ under $H_0$. 
Write $Z_i=X_i-\mu_x$ and $W_j=Y_j-\mu_y$ so that $E[Z_i]=E[W_j]= E[D] = 0$ and 
$E[Z_i^2]=\sigma_x^2$, $E[W_j^2]=\sigma_y^2$, 
$E[Z_i^3]=\gamma_x\sigma_x^3$, and $E[W_j^3]=\gamma_y\sigma_y^3$.
Then
\[
\bar X-\mu_x = \frac{1}{n_x}\sum_{i=1}^{n_x} Z_i,
\qquad
\bar Y-\mu_y = \frac{1}{n_y}\sum_{j=1}^{n_y} W_j.
\]
Expanding the cube,
To compute its third moment, expand the cube:
\[
\Biggl(\sum_{j=1}^{n_y} W_j\Biggr)^3
= \sum_{j=1}^{n_y} W_j^3
+ 3\sum_{\substack{j_1,j_2=1 \\ j_1\ne j_2}}^{n_y} W_{j_1}^2 W_{j_2}
+ 6\sum_{1\le j_1<j_2<j_3\le n_y} W_{j_1}W_{j_2}W_{j_3}.
\]
Taking expectations of each term, we have 
\[
E[W_{j_1}^2W_{j_2}] = E[W_{j_1}^2]E[W_{j_2}] = \sigma_y^2 \cdot 0 = 0, 
\]
\[
E[W_{j_1}W_{j_2}W_{j_3}] = E[W_{j_1}]E[W_{j_2}]E[W_{j_3}] = 0.
\]
so this sum also vanishes.
Hence,
\[
E\!\left[\Biggl(\sum_{j=1}^{n_y} W_j\Biggr)^3\right]
= n_y \gamma_y\sigma_y^3.
\]
Scaling by $n_y^{-3}$ we obtain
\[
E\!\left[(\bar Y-\mu_y)^3\right]
= \frac{1}{n_y^3}\cdot n_y \gamma_y\sigma_y^3
= \frac{\gamma_y\sigma_y^3}{n_y^2}.
\]
Analogously,
\[
E\!\left[(\bar X-\mu_x)^3\right]
= \frac{\gamma_x\sigma_x^3}{n_x^2}.
\]
Since $\bar X$ and $\bar Y$ are independent, all cross terms vanish when forming $D=\bar Y-\bar X$, and thus
\[
E[D^3] = E[(\bar Y-\mu_y)^3] - E[(\bar X-\mu_x)^3]
= \frac{\gamma_y\sigma_y^3}{n_y^2} - \frac{\gamma_x\sigma_x^3}{n_x^2}.
\]
Furthermore, we have the variance of $D$ is 
$\sigma_D^2 = \sigma_x^2/n_x + \sigma_y^2/n_y$.
Therefore, the skewness of mean difference is 
\begin{equation*}
    \gamma_D = \frac{E(D - ED)^3}{\sigma_D^3} = \frac{\gamma_y \sigma_y^3/n_y^2 - \gamma_x\sigma_x^3/n_x^2}{\sigma_D^3}
\end{equation*}
In special case of $\sigma_x = \sigma_y = \sigma$, we have
\begin{equation*}
    \gamma_D = \frac{\gamma_y - k^2\gamma_x^2}{(1+k)\sqrt{Nk}}.
\end{equation*}

Similarly, consider the fourth power. Expanding yields
\begin{align*}
E\!\left[(\bar Y-\mu_y)^4\right] 
&= \frac{1}{n_y^4}  E\!\left[\Bigl(\sum_{j=1}^{n_y} W_j\Bigr)^4\right] \\
&= \frac{1}{n_y^4} \Biggl(
\sum_{j=1}^{n_y} E[W_j^4] 
+ \sum_{j\ne j'} E[W_j^2 W_{j'}^2] \notag\\
&\quad\ + \sum_{j\ne j'} E[W_j^3 W_{j'}] 
+ \sum_{j\ne j'\ne j''} E[W_j^2 W_{j'} W_{j''}]
\Biggr). \label{eq:fourthmomentY}
\end{align*}
Since $E[W_j^4]=\tau_y\sigma_y^4$, $E[W_j^2 W_{j'}^2]=\sigma_y^4$ for $j\ne j'$ and the odd power terms expectation vanish, we obtain
\begin{align*}
E\!\left[(\bar Y - \mu_y)^4\right]
&= \frac{1}{n_y^4}E\!\left[\Bigl(\sum_{j=1}^{n_y} W_j\Bigr)^4\right] \\
&= \frac{1}{n_y^4}\Bigl( n_y E[W_1^4] 
+ 3 n_y(n_y-1)E[W_1^2]^2 \Bigr) \\
&= \frac{1}{n_y^4}\Bigl( n_y \tau_y\sigma_y^4 
+ 3 n_y(n_y-1)\sigma_y^4 \Bigr) \\
&= \frac{\tau_y+3(n_y-1)}{n_y^3}\sigma_y^4.
\end{align*}
An identical calculation for the control group, with 
$Z_i = X_i - \mu_x$ and kurtosis $\tau_x$, yields
\begin{equation*}
E\!\left[(\bar X - \mu_x)^4\right]
= \frac{\tau_x+3(n_x-1)}{n_x^3}\sigma_x^4.
\end{equation*}
Since
\begin{equation*}
D^4 = (A-B)^4 = A^4 - 4A^3B + 6A^2B^2 - 4AB^3 + B^4.
\end{equation*}
Because $E[A]=E[B]=0$ and $A,B$ are independent, we have
\[
E[A^3B]=E[A^3]E[B]=0, \quad 
E[AB^3]=E[A]E[B^3]=0, \]
Therefore,
\begin{equation*}
E[D^4] = E[A^4] + E[B^4] + 6E[A^2]E[B^2].
\end{equation*}
Substituting the expressions above,
\begin{align*}
E[D^4]
&= \frac{\tau_y+3(n_y-1)}{n_y^3}\sigma_y^4
+ \frac{\tau_x+3(n_x-1)}{n_x^3}\sigma_x^4
+ 6\frac{\sigma_y^2}{n_y}\frac{\sigma_x^2}{n_x}.
\end{align*}
Note that
\[
\frac{3\sigma_y^4}{n_y^2} + \frac{3\sigma_x^4}{n_x^2} 
+ 6\frac{\sigma_y^2}{n_y}\frac{\sigma_x^2}{n_x}
= 3\Bigl(\frac{\sigma_y^2}{n_y}+\frac{\sigma_x^2}{n_x}\Bigr)^2.
\]
Hence we can rewrite
\begin{equation*}
E[D^4]
= 3\Bigl(\frac{\sigma_y^2}{n_y}+\frac{\sigma_x^2}{n_x}\Bigr)^2
+ \frac{(\tau_y-3)\sigma_y^4}{n_y^3}
+ \frac{(\tau_x-3)\sigma_x^4}{n_x^3}.
\end{equation*}
Under $H_0$ we have $E[D]=0$, so the kurtosis of $D$ is
\begin{equation*}
\tau_D
= \frac{E[D^4]}{\sigma_D^4}
= 3 + \frac{(\tau_y-3)\sigma_y^4/n_y^3 + (\tau_x-3)\sigma_x^4/n_x^3}{\bigl(\sigma_y^2/n_y+\sigma_x^2/n_x\bigr)^2}.
\end{equation*}

If $\sigma_x=\sigma_y=\sigma$ and $n_y=kn_x$ for allocation ratio $k>0$, then
\begin{equation*}
\tau_D
= 3 + \frac{(\tau_y-3) + (\tau_x-3)k^3}{N(1+k)k}.
\end{equation*}

\end{proof}

\begin{proof}[Proof of Theorem~\ref{thm:deviation}]
Under the stated regularity conditions, the distribution function of the studentized Welch $t$ statistic
\[
T \;=\; \frac{\bar Y-\bar X}{\sqrt{\hat\sigma_x^2/n_x+\hat\sigma_y^2/n_y}}
\]
admits the Edgeworth expansion and uniform in $z$ \citep{hall2013bootstrap},
\begin{equation}
\label{equ:cdfee}
P(T\le z)
= \Phi(z)
+ \phi(z)\{q_1(z)
+ q_2(z)\}
+ O(N^{-3/2}),
\end{equation}
For a studentized mean, the correction polynomials are
\[
q_1(z)=\frac{\gamma_D}{6}(2z^2+1),
\]
\begin{align*}
q_2(z) &= \frac{\tau_D-3}{12}(z^3 - 3z)
- \frac{\gamma_D^2}{18}(z^5 + 2z^3 - 3z) \\
&\quad -\frac{(1+k)}{4N}\;
\frac{\bigl(k^3\sigma_x^4+\sigma_y^4\bigr)(z^3+3z)+2k(1+k)\sigma_x^2\sigma_y^2z}
{k\bigl(k\sigma_x^2+\sigma_y^2\bigr)^2},
\end{align*}
where $\gamma_D$ and $\tau_D$ are the skewness and 
kurtosis of the mean difference $D=\bar Y-\bar X$.
Using \eqref{equ:cdfee}, the right–tail error is
\begin{align*}
\alpha_R^*-\frac{\alpha}{2}
&= P(T\geq z_{1-\alpha/2})-\frac{\alpha}{2}\\
&= -\,\phi(z_{\alpha/2})\big[q_1(z_{\alpha/2})-q_2(z_{\alpha/2})\big]
+ O(N^{-3/2}).
\end{align*}
and, by symmetry of \eqref{equ:cdfee},
\[
\alpha_L^*-\frac{\alpha}{2}
= \phi(z_{\alpha/2})\big[q_1(z_{\alpha/2}) +q_2(z_{\alpha/2})\big]
  + O(N^{-3/2}).
\]
Since $\phi(x)$ is even and $z_{1-\alpha/2} = -z_{\alpha/2}$, the total deviations satisfy
\[
\Delta(\alpha) = |\alpha_L^* + \alpha_R^* - \alpha| = |q_2(z_\alpha)\phi(z_\alpha)| +O(N^{-3/2})
\]
From Eq.~\eqref{equ:gammad}, we have that $q_2(z)$ is of order $O(N^{-1})$. 
Hence, 
\begin{align*}
\alpha_L^* - \frac{\alpha}{2}
  &= \frac{\gamma_D}{6}\,\phi(z_{\alpha/2})(2z_{\alpha/2}^2+1)
     + O(N^{-1}),\\[4pt]
\alpha_R^* - \frac{\alpha}{2}
  &= -\frac{\gamma_D}{6}\,\phi(z_{\alpha/2})(2z_{\alpha/2}^2+1)
     + O(N^{-1}).
\end{align*}
and 
\[
\Delta(\alpha)
= O(N^{-1}).
\]
\end{proof}

\begin{proof}[Proof of Theorem~\ref{thm:threshold}]
Under the given conditions, the Edgeworth expansion \citep{hall2013bootstrap} for the test statistic $T$  is
\begin{equation*}
P(T \le z) = \Phi(z) + \phi(z)q_1(z) + O(N^{-1}).
\end{equation*} 
Considering the tail probability at $z = z_{1-\alpha/2}$, the error term is 
\begin{align*}
 |P(T \leq z_{1-\alpha/2}) - \Phi(z_{1-\alpha/2})| 
 &= |\frac{\gamma_D}{6}(2z_{1-\alpha/2}^2+1)\phi(z_{1-\alpha/2}) + O(N^{-1})| \\
&\leq \frac{a_1}{\sqrt{N}} + O(N^{-1})   
\end{align*}
The leading error is the first term with order at $O(N^{-1/2})$. 
Let $N_{\min}^{(1)} = (a_1/\epsilon)^2$. Then for $N\geq N_{\min}^{(1)}$, we have $a_1/\sqrt{N}\leq \epsilon$, and thus
\begin{equation*}
|P(T \le z_{1-\alpha/2}) - \Phi(z_{1-\alpha/2})| \leq \epsilon + O(N^{-1}) = \epsilon + O(\epsilon^2).
\end{equation*}
By symmetry, the same bound holds for the other tail. Therefore,
\begin{equation*}
\max\{|\alpha_L^*-\frac{\alpha}{2}|, |\alpha_R^*-\frac{\alpha}{2}|\} \leq \epsilon + O(\epsilon^2),
\end{equation*}
completing the proof of part (i).

When the sixth moments exist, the Edgeworth expansion \citep{hall2013bootstrap} becomes
\begin{equation*}
P(T\le z)
= \Phi(z)
+ \phi(z)\{q_1(z)
+ q_2(z)\}
+ O(N^{-3/2}),
\end{equation*}
The errors in the tail probabilities are
\begin{equation*}
|\alpha_L^*- \alpha/2| = |a_2/N + a_1/\sqrt{N}|,\quad    |\alpha_R^*- \alpha/2| = |a_2/N - a_1/\sqrt{N}|,
\end{equation*}
where $a_1$ and $a_2$ are constants determined by the skewness and kurtosis terms as defined in the theorem.
Let $u = N^{-1/2}$. To ensure that both tail errors remain below a given tolerance $\epsilon$, we require
\begin{equation}
\label{equ:two}
|a_2 u^2 + a_1 u| \leq \epsilon, \qquad |a_2 u^2 - a_1 u| \leq \epsilon.
\end{equation}
Solving \eqref{equ:two} gives the admissible range of $u$ as
$0 < u \le (\sqrt{a_1^2 - 4|a_2|\epsilon} - |a_1|)/(2|a_2|)$ when $a_1^2 - 4|a_2|\epsilon \ge 0$, and $0 < u \le (\sqrt{a_1^2 + 4|a_2|\epsilon} - |a_1|)/(2|a_2|)$ otherwise.

Combining both cases, a sufficient bound on the sample size is
\begin{equation*}
N \geq N_{\min}^{(2)} = \left(
  \frac{|a_1| + \sqrt{\,a_1^2 - 4|a_2|\epsilon \operatorname{sign}\bigl(a_1^2 - 4|a_2|\epsilon\bigr)}}
       {2\epsilon}
  \right)^{2}.
\end{equation*}
Therefore, whenever $N \geq N_{\min}^{(2)}$, both tails satisfy
\[
\max\{|\alpha^*_R-\alpha/2|,\,|\alpha^*_L-\alpha/2|\}\leq \epsilon + O(\epsilon^3),
\]
which guarantees that each tail error remains within the prescribed tolerance $\epsilon$.

\end{proof}

\end{document}